\documentclass[letterpaper, 10 pt, conference]{ieeeconf}  

\IEEEoverridecommandlockouts                              

\let\labelindent\relax

\usepackage{graphics}
\usepackage{amsfonts}
\usepackage{amsmath}
\usepackage{dsfont}
\usepackage{bm}
\usepackage{upgreek}

\usepackage{amssymb}
\usepackage{leftindex}
\usepackage{algorithm}
\usepackage{algpseudocode}
\usepackage{epsfig}
\usepackage{graphicx}
\usepackage{svg}
\usepackage{caption}
\usepackage{mathrsfs}
\usepackage{array}
\usepackage{multirow}
\usepackage{multicol}
\usepackage{booktabs}
\usepackage{makecell}
\usepackage{xcolor}
\usepackage{amsfonts}
\usepackage{flushend}
\usepackage{soul}
\usepackage[colorlinks, linkcolor=black]{hyperref}
\usepackage{stfloats}
\usepackage{balance}
\usepackage{threeparttable}
\usepackage{enumitem}
\usepackage{bbding}
\usepackage{pifont}
\usepackage{framed}

\usepackage{amsthm}

\newcommand\norm[1]{\lVert #1 \rVert}
\renewenvironment{proof}[1][\proofname]{\par\noindent\textbf{#1.} }{\hfill$\square$\par}

\newtheorem{theorem}{Theorem}
\newtheorem{lemma}[theorem]{Lemma}

\title{\LARGE \bf
Toward Embodiment Equivariant Vision-Language-Action Policy
}

\author{Anzhe Chen, Yifei Yang, Zhenjie Zhu, Kechun Xu, Zhongxiang Zhou, Rong Xiong, Yue Wang$^*$
}

\begin{document}

\maketitle
\thispagestyle{empty}
\pagestyle{empty}

\bstctlcite{IEEEexample:BSTcontrol}


\begin{abstract}
    Vision-language-action policies learn manipulation skills across tasks, environments and embodiments through large-scale pre-training. However, their ability to generalize to novel robot configurations remains limited. Most approaches emphasize model size, dataset scale and diversity while paying less attention to the design of action spaces. This leads to the configuration generalization problem, which requires costly adaptation. We address this challenge by formulating cross-embodiment pre-training as designing policies equivariant to embodiment configuration transformations. Building on this principle, we propose a framework that (i) establishes a embodiment equivariance theory for action space and policy design, (ii) introduces an action decoder that enforces configuration equivariance, and (iii) incorporates a geometry-aware network architecture to enhance embodiment-agnostic spatial reasoning. Extensive experiments in both simulation and real-world settings demonstrate that our approach improves pre-training effectiveness and enables efficient fine-tuning on novel robot embodiments. Our code is available at \url{https://github.com/hhcaz/e2vla}
\end{abstract}


\section{Introduction}

Vision-language-action (VLA) policies have attracted increasing attention in recent years. These policies are typically pre-trained on large-scale cross-embodiment datasets that span diverse tasks and environments, and then fine-tuned for specific embodiments and downstream tasks. While such pre-training has shown effectiveness, it often requires the target embodiment to share similar robot configurations with those seen during pre-training. Consequently, when faced with novel embodiments, the policy tends to fail, leading to considerable additional training costs for adaptation.

Ideally, a VLA policy should exhibit a certain degree of zero-shot generalization to unseen robot configurations (e.g., variations in the coordinate definitions of the robot base and end-effector) after pre-training. Such capability would enable efficient fine-tuning with minimal additional data and training. However, existing approaches \cite{rt1, RT-2,rdt,pi0,openvla} to VLA pre-training focus primarily on scaling up parameters, datasets and task diversity, while paying limited attention to the design of neural architectures and action spaces. 
leaving the configuration generalization problem unresolved. 
Without configuration generalization, the policy may tend to learn the embodiment-specific knowledge during pre-training, which hinders the sharing of task-level knowledge and reduces the overall effectiveness of cross-embodiment pre-training.

A key requirement for cross-embodiment pre-training, therefore, is to prevent networks from hard-memorizing embodiment-specific details. This motivates the question: \textit{Is there a framework to unifies action space design across embodiments to learn task-level knowledge but not embodiment-specific kinematics details?}

\begin{figure}[t]
\includegraphics[width=\linewidth]{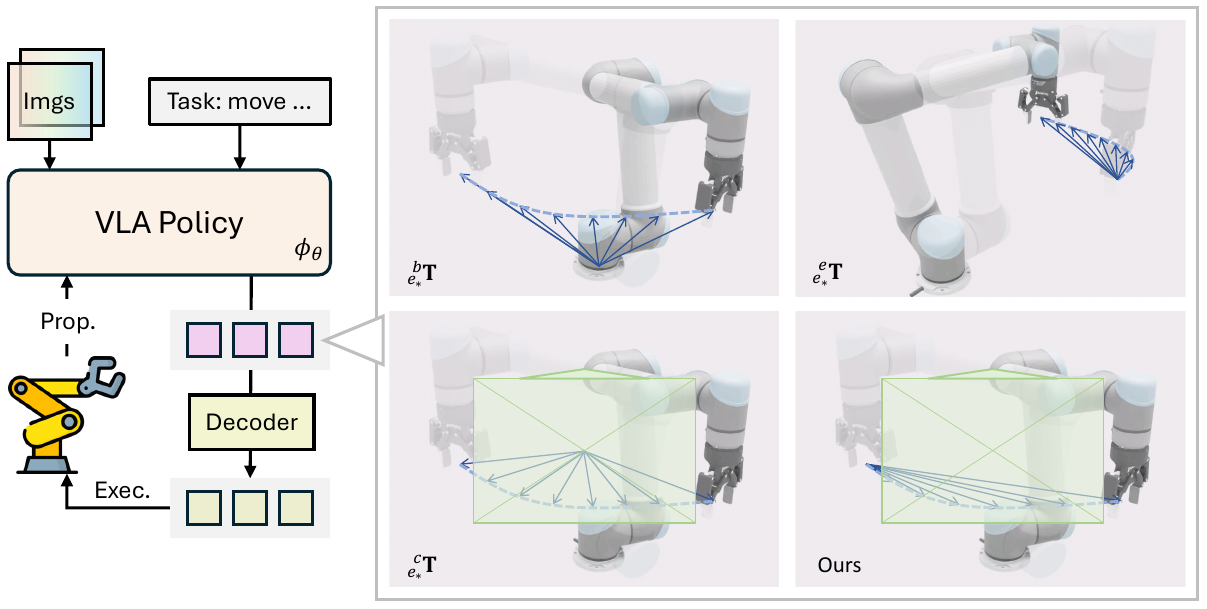}
\centering
\caption{Different action space design for VLA policy learning. We learn the relative end-effector motion projected in camera space.}
\label{fig: teaser}
\end{figure}

We draw inspiration from equivariance, where neural networks are endowed with symmetry properties by design rather than through learning from data. For example, $\mathrm{E}(2)$-equivariant CNN \cite{e2cnn} generalizes this property to rotations and reflections in the Euclidean plane. It is applied in medical imaging, satellite imagery where objects may appear at arbitrary orientations. $\mathrm{SE}(3)$-Transformers \cite{se3transformer} and Equivariant Graph Neural Networks \cite{engnn} ensure equivariance under 3D rotations and translations, which are applied in protein structure prediction, molecular dynamics modeling and 3D perception tasks. Compared with data augmentation, equivariant models encode inductive biases more efficiently, leading to reduced model size and improved generalization to unseen transformations. 

We hypothesize that the limited generalization ability of existing VLA policies stems from their use of joint-space actions or end-effector pose in base (Fig. \ref{fig: teaser}), which makes the overall policy non-equivariant to embodiment configuration transformations and thus leads to overfitting to specific embodiments during pre-training. Therefore, we propose to formulate cross-embodiment learning as the design of policy equivariant to embodiment configuration transformations. Such a formulation enables systematic generalization to novel embodiments and makes pre-training more effective. Our main contributions are as follows:
\begin{itemize}
    \item Policy equivariance theory that unifies action space design across embodiments;
    \item An action decoder that is equivariant to embodiment configuration transformations;
    \item A network architecture that enhances geometric reasoning within the equivariant policy;
    \item Simulation and real-world validation of effective pre-training and efficient fine-tuning with our approach.
\end{itemize}

\section{Related Works}

\textbf{VLA Policies.}
Vision-language-action (VLA) policies represent a key step toward unifying perception, language understanding, and action execution within a single framework. Classic policies such as RT-1 \cite{rt1}, ACT \cite{ACT}, and Octo \cite{Octo} adopt transformer-based architectures trained on diverse robot trajectories across multiple environments and tasks. More recent approaches leverage vision-language models (VLMs) pretrained on internet-scale data, adopting paradigms similar to large language models (LLMs). For instance, RT-2 \cite{RT-2} and OpenVLA \cite{openvla} treat actions as tokens and employ next-token prediction to generate robot behaviors. While these methods improve generalization and instruction-following capabilities through enhanced language understanding and visual grounding, their inference speed remains insufficient for real-time robotic systems.

To address this, dual-system architectures have been proposed. RoboDual \cite{robodual}, Fis-VLA \cite{fis-vla}, and OpenHelix \cite{openhelix} decouple reasoning and control by running a slow reasoning module alongside a fast action expert, thereby combining the strengths of both. In addition, diffusion-based methods have been explored to better capture multi-modal human demonstrations. Diffusion Policy \cite{diffusion_policy}, together with flow-based action generation methods \cite{rdt, pi0, dexvla, diffusionvla}, has become a dominant trend in recent VLA research.

Since robot action datasets remain relatively small compared to internet-scale vision-language corpora, several works have introduced auxiliary tasks to leverage large-scale image and video data. SuSIE \cite{susie} generates subgoal images through image editing, while the GR series \cite{gr1, gr2, gr3} exploit video generation pre-training for manipulation. Other methods avoid explicit future image prediction and instead generate features, such as keypoint motions \cite{atm} or latent visual embeddings \cite{univla}, to better utilize internet-scale supervision.

\textbf{Action Space.}
Open-source robot datasets \cite{oxe} contain heterogeneous action representations, posing challenges for cross-embodiment pre-training. RDT-1B \cite{rdt} directly predicts high-dimensional raw action vectors containing mixed quantities such as joint positions, end-effector poses, and deltas. $\pi_0$ \cite{pi0} learns joint-space or end-effector actions for most datasets, but a fixed index in the action vector may correspond to different physical meanings across robots (e.g., index-6 representing gripper width in a 6-DoF arm versus a wrist joint in a 7-DoF arm). DexVLA \cite{dexvla} addresses this by employing multiple action heads tailored for different embodiments.

Several works attempt to design unified action spaces for cross-embodiment learning. For example, \cite{pushing-the-limit} proposes ego-centric actions applicable to both navigation and manipulation robots, but requires manual alignment across datasets to ensure consistency in action semantics. ATM \cite{atm} predicts pixel-level motions and employs an inverse dynamics model to map them to executable actions. UniVLA \cite{univla} predicts future latent features of observations and derive task-centric action representations with a latent action model. While effective, these approaches rely on over-parameterized image-space motion representations and require training an additional inverse dynamics model for each new embodiment.

\section{Embodiment Equivariance}

\subsection{Problem Formulation}

\textbf{Definitions.} We define embodiment configuration as the combinition of its end-effector pose under base and camera pose under base:
$\mathbf{m} \in \mathcal{M} = \{[ \leftindex[]^b_e \mathbf{T},  \leftindex[]^b_c \mathbf{T} ] \}$, 
and define action as the desired end-effector pose under base:
$
    \mathbf{a} \in \mathcal{A} = \{ \leftindex[]^b_{e_*} \mathbf{T} \}
$. 
Given image observation $\mathcal{I}$ and language instruction $\mathcal{L}$, a VLA policy $\pi_\theta$ maps:
\begin{equation}
    \mathbf{a} = \pi_\theta ( \mathbf{m}, \mathcal{I}, \mathcal{L} )
\end{equation}
Let $\mathcal{G} = \{ [\leftindex[]^{b'}_b \mathbf{T}, \leftindex[]^e_{e'} \mathbf{T}] \}$ denote the embodiment transformation group that transforms the existing robot base and end-effector definitions to the new ones (Fig. \ref{fig: g_tform}). Given $\mathbf{g} \in \mathcal{G}$, we have:
\begin{equation}
    \begin{aligned}
        &\mathbf{g} \circ \mathbf{m} = [ \leftindex[]^{b'}_b \mathbf{T} \leftindex[]^b_e \mathbf{T} \leftindex[]^e_{e'} \mathbf{T}, \leftindex[]^{b'}_b \mathbf{T} \leftindex[]^b_c \mathbf{T} ] = [\leftindex[]^{b'}_{e'} \mathbf{T}, \leftindex[]^{b'}_c \mathbf{T}] \\
        &\mathbf{g} \circ \mathbf{a} = \leftindex[]^{b'}_b \mathbf{T} \leftindex[]^b_{e_*} \mathbf{T} \leftindex[]^e_{e'} \mathbf{T} = \leftindex[]^{b'}_{e'_*} \mathbf{T} 
    \end{aligned}
\end{equation}
where $\circ$ denotes the group action operator.

\begin{figure}[t]
\includegraphics[width=\linewidth]{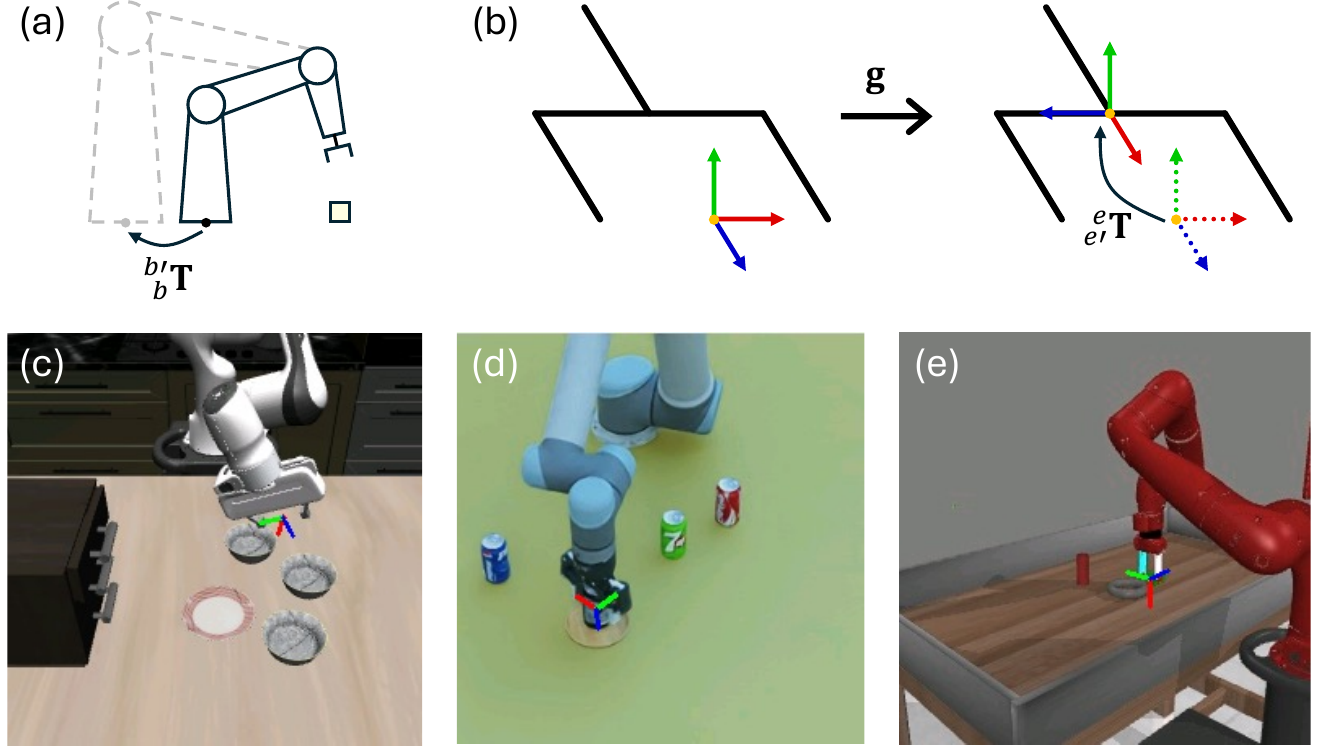}
\centering
\caption{Schematics of the pose coordinate definition and group transformation of embodiment base and end-effector. (a) Embodiment with different base definition should have similar relative pose trajectory between end-effector pose and the objects when completing similar tasks. (b) Different coordinate definitions with respect to grippers with similar topologies. Some define the origin at the center of fingertip, and some may defines it at the install link of gripper. The orientations may also be different. (c)(d)(e) Example coordinate definitions visualized from different datasets with different embodiments.}
\label{fig: g_tform}
\end{figure}

\textbf{Embodiment Equivariant Policy}.
We desire the policy to be inherently equivariant to configuration transformations rather than learning through data:
\begin{equation}
    \pi_\theta ( \mathbf{g} \circ \mathbf{m}, \mathcal{I}, \mathcal{L} ) = \mathbf{g} \circ \pi_\theta ( \mathbf{m}, \mathcal{I}, \mathcal{L} )
\end{equation}
With equivariance ensured, the policy can focus on task-level knowledge rather than embodiment-specific details, enabling more effective pre-training and more efficient fine-tuning.

\subsection{Architecture Design}
Equivariance is challenging to achieve with purely learning-based methods. To address this, we decompose the policy into two parts: a learning based policy $\phi_\theta$ that is invariant to configuration transformation and an analytical decoder $\mathcal{D}$ that is equivariant to configuration transformation:
\begin{equation}
    \pi_\theta(\mathbf{m}, \mathcal{I}, \mathcal{L}) = \mathcal{D}(\mathbf{m},  \phi_\theta(\mathbf{m}, \mathcal{I}, \mathcal{L}))
\end{equation}
where the following constraints are desired:
\begin{equation}
\label{eqn: constraints 1}
    \begin{aligned}
        & \phi_\theta(\mathbf{g} \circ \mathbf{m}, \mathcal{I}, \mathcal{L}) = \phi_\theta(\mathbf{m}, \mathcal{I}, \mathcal{L}) \\
        & \mathcal{D}(\mathbf{g} \circ \mathbf{m}, \phi_\theta ) = \mathbf{g} \circ \mathcal{D}(\mathbf{m}, \phi_\theta)
    \end{aligned}
\end{equation}

\begin{lemma}
If $\mathcal{D}$ is equivariant to $\mathcal{G}$ and $\phi_\theta$ is invariant to $\mathcal{G}$, then $\pi_\theta$ is equivariant to $\mathcal{G}$.
\end{lemma}
\begin{proof}
    This result follows directly from the constraints in Eq. \ref{eqn: constraints 1}.
\end{proof}
\vspace{0.8em}

A naive design that satisfies these constraints is:
\begin{equation}
\label{eq: D design}
    \begin{aligned}
        & \phi_\theta( \mathbf{m}, \mathcal{I}, \mathcal{L}) = \phi_\theta(\mathcal{I}, \mathcal{L}) \\
        & \mathcal{D}( [ \leftindex[]^b_c \mathbf{T}, \leftindex[]^b_{e} \mathbf{T}], \phi_\theta ) = \leftindex[]^b_c \mathbf{T} \ \phi_\theta \leftindex[]^b_c \mathbf{T}^{-1} \leftindex[]^b_{e} \mathbf{T}
    \end{aligned}
\end{equation}

\begin{theorem}
    If the network and decoder satisfy Eq. \ref{eq: D design}, the policy is equivariant to configuration transformations.
\end{theorem}

\begin{proof}
    In Eq. \ref{eq: D design}, $\phi_\theta( \mathbf{m}, \mathcal{I}, \mathcal{L}) = \phi_\theta(\mathcal{I}, \mathcal{L})$ means the neural network discards all the embodiment information, therefore $\phi_\theta$ is obviously invariant to $\mathcal{G}$. For decoder $\mathcal{D}$, we have:
    \begin{equation}
        \begin{aligned}
            \mathcal{D}(\mathbf{g} \circ \mathbf{m}, \phi_\theta) &= \leftindex[]^{b'}_c \mathbf{T} \ \phi_\theta \leftindex[]^{b'}_{c'} \mathbf{T}^{-1} \leftindex[]^{b'}_{e} \mathbf{T} \\
            &= \leftindex[]^{b'}_b \mathbf{T} (\leftindex[]^b_c \mathbf{T} \ \phi_\theta \leftindex[]^b_c \mathbf{T}^{-1} \leftindex[]^b_{e} \mathbf{T}) \leftindex[]^{e}_{e'} \mathbf{T} \\
            &= \mathbf{g} \circ \mathcal{D}(\mathbf{m}, \phi_\theta)
        \end{aligned}
    \end{equation}
    Therefore, $\mathcal{D}$ is equivariant to $\mathcal{G}$. Finally, we have:
    \begin{equation}
        \begin{aligned}
            \pi_\theta(\mathbf{g} \circ \mathbf{m}, \mathcal{I}, \mathcal{L}) &= \mathcal{D}(\mathbf{g} \circ \mathbf{m}, \phi_\theta( \mathbf{g} \circ \mathbf{m}, \mathcal{I}, \mathcal{L} )) \\
            &= \mathbf{g} \circ \mathcal{D}(\mathbf{m}, \phi_\theta(\mathbf{m}, \mathcal{I}, \mathcal{L})) \\
            &= \mathbf{g} \circ \pi_\theta(\mathbf{m}, \mathcal{I}, \mathcal{L})
        \end{aligned}
    \end{equation}
    Therefore, the policy $\pi_\theta$ is equivariant to $\mathcal{G}$.
\end{proof}

\vspace{0.8em}
Note discarding all the embodiment configuration in $\phi_\theta$ is a naive design that yields poor performance. We will discuss how to preserve embodiment configurations while maintaining invariance in section \ref{sec: geo util}. Given the decoder design in Eq. \ref{eq: D design}, we would see the learning objective $\mathbf{y}$ of $\phi_\theta$ is:
\begin{equation}
    \phi_\theta \rightarrow \mathbf{y} = \leftindex[]^b_c \mathbf{T}^{-1} \leftindex[]^b_{e_*} \mathbf{T} \leftindex[]^b_{e} \mathbf{T}^{-1} \leftindex[]^b_c \mathbf{T} = \leftindex[]^c_{e_*} \mathbf{T} \leftindex[]^e_c \mathbf{T}
\end{equation}
It is reasonable to expect a configuration invariant policy $\phi_\theta$ to predict an also configuration invariant objective $\mathbf{y}$.

\subsection{Trade-off for Positioning Precision}

In real-world dataset, the end-effector poses are reliable as they are directly read from robot SDKs. However, the cameras are often not well calibrated, introducing errors in camera extrinsic. While the decoder design in Eq. \ref{eq: D design} achieves full equivariance, its learning objective $\mathbf{y}$ is sensitive to such calibration errors, especially in the translation components. We could reformulate the translation component of the original learning objective to:
\begin{equation}
\label{eq: yt}
    \mathbf{y}^t = \leftindex[]^c_{e_*} \mathbf{R} \leftindex[]^e_c \mathbf{t} + \leftindex[]^c_{e_*} \mathbf{t} = \leftindex[]^c_e \mathbf{R} ( \leftindex[]^e_{e_*} \mathbf{R} - \mathbf{I} ) \leftindex[]^e_c \mathbf{t} + \leftindex[]^c_e \mathbf{R} \leftindex[]^e_{e_*} \mathbf{t}
\end{equation}
In Eq. \ref{eq: yt}, we decompose $\mathbf{y}_t$ into relative end-effector pose related components (accurate $\leftindex[]^e_{e_*} \mathbf{R}$ and $\leftindex[]^e_{e_*} \mathbf{t}$) and camera extrinsic related components (inaccurate $\leftindex[]^c_e \mathbf{R}$ and $\leftindex[]^e_c \mathbf{t}$). We notice that in most cases we have $\norm{\leftindex[]^e_c \mathbf{t}} > \norm{ \leftindex[]^e_{e_*} \mathbf{t} }$, and the extrinsic calibration error will be amplified by $\norm{\leftindex[]^e_c \mathbf{t}}$, yielding large translation error which may overwhelm the norm of ground truth learning objective $\mathbf{y}^t$ itself and make the pre-training less effective. Therefore, we remove the $\leftindex[]^e_c \mathbf{t}$ related item, and refine the decoder and the corresponding learning objective to:
\begin{equation}
\label{eq: D design stable}
\begin{aligned}
    &\mathcal{D}_r( [ \leftindex[]^b_c \mathbf{T}, \leftindex[]^b_{e} \mathbf{T}], \phi_\theta ) = \begin{bmatrix}
        \leftindex[]^b_c \mathbf{R} \phi_\theta^R \leftindex[]^b_c \mathbf{R}^{-1} \leftindex[]^b_{e} \mathbf{R} & \leftindex[]^b_c \mathbf{R} \phi_\theta^t + \leftindex[]^b_e \mathbf{t} \\
        \mathbf{0} & 1
    \end{bmatrix} \\
    &\phi_\theta \rightarrow \mathbf{y}_r = \begin{bmatrix}
        \leftindex[]^c_{e_*} \mathbf{R} \leftindex[]^e_c \mathbf{R} & \leftindex[]^c_e \mathbf{R} \leftindex[]^e_{e_*} \mathbf{t} \\
        \mathbf{0} & 1
    \end{bmatrix}
\end{aligned}
\end{equation}
The revised decoder in Eq. \ref{eq: D design stable} remains equivariant to $\mathcal{G}$ except for pure end-effector translations transformation group $\mathcal{T}=\{ \leftindex[]^e_{e'} \mathbf{t} \}$, but achieves much greater robustness in practice. All subsequent experiments are based on this revised design.

\subsection{Discussion of Existing Policy Learning}

Beyond our proposed formulation, several alternative action spaces are commonly used in VLA policies:
\begin{equation}
    \begin{aligned}
        \text{BE:} \quad & \phi_\theta(\mathbf{m, \mathcal{I}, \mathcal{L}}) \rightarrow \leftindex[]^b_{e_*} \mathbf{T}, &&\mathcal{D}(\mathbf{m}, \phi_\theta) = \phi_\theta \\
        \text{EE:} \quad & \phi_\theta(\mathbf{m, \mathcal{I}, \mathcal{L}}) \rightarrow \leftindex[]^e_{e_*} \mathbf{T}, &&\mathcal{D}(\mathbf{m}, \phi_\theta) = \leftindex[]^b_e \mathbf{T} \ \phi_\theta  \\
        \text{CE:} \quad & \phi_\theta(\mathbf{m, \mathcal{I}, \mathcal{L}}) \rightarrow \leftindex[]^c_{e_*} \mathbf{T}, &&\mathcal{D}(\mathbf{m}, \phi_\theta) = \leftindex[]^b_c \mathbf{T} \ \phi_\theta \\ 
        \text{IM:} \quad & \phi_\theta(\mathbf{m, \mathcal{I}, \mathcal{L}}) \rightarrow \mathcal{I}_*, &&\mathcal{D}(\mathbf{m}, \phi_\theta) = \psi_\theta(\mathcal{I}, \phi_\theta)
    \end{aligned}
\end{equation}

Denoting $\mathcal{G}_b = \{ \leftindex[]^{b'}_b \mathbf{T} \}$ and $\mathcal{G}_e = \{ \leftindex[]^e_{e'} \mathbf{T} \}$ separately as base transformation group and end-effector transformation group, we could draw the following conclusions:

\noindent \textbf{1)} BE predicts the absolute end-effector pose in the base frame. The decoder is invariant to $\mathcal{G}$, forcing the policy to learn equivariance through data.

\noindent \textbf{2)} EE predicts the relative end-effector pose, which is adopted by OpenVLA \cite{openvla} and $\pi_0$ \cite{pi0} when pre-training on a subset of OXE dataset \cite{oxe}. The decoder is equivariant to $\mathcal{G}_b$, however, is neither equivariant nor invariant to $\mathcal{G}_e$. Therefore, the policy is not equivariant to $\mathcal{G}$.

\noindent \textbf{3)} CE predicts the end-effector pose in camera frame, which is adopted by RISE \cite{rise}. The decoder is equivariant to $\mathcal{G}_b$ and invariant to $\mathcal{G}_e$, which requires the $\phi_\theta$ to be invariant to $\mathcal{G}_b$ and equivariant to $\mathcal{G}_e$. While the former can be easily achieved by reformulating $\mathbf{m}$ to $\leftindex[]^c _e \mathbf{T}$, the latter is not guaranteed by network structure but learning through data.

\noindent \textbf{4)} IM predicts the dense/sparse future image space features (e.g., keypoints in ATM \cite{atm} or latent embeddings in UniVLA \cite{univla}). The encoder is invariant to $\mathcal{G}$, however, the decoder is a neural network without guaranteed equivariance.

These comparisons highlight the limitations of existing formulations and motivate our design of a theoretically grounded, embodiment transformation equivariant policy.

\begin{figure}[t]
\includegraphics[width=0.8\linewidth]{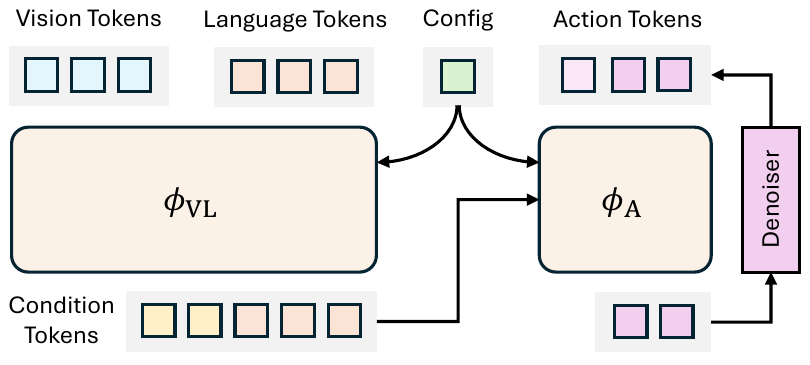}
\centering
\caption{Schematics of model $\phi$. Model $\phi$ is composed of $\phi_\text{VL}$ and $\phi_\text{A}$. $\phi_\text{VL}$ generate geometric-aware and task-related vision-language condition tokens. $\phi_\text{A}$ generate actions with conditions from $\phi_\text{VL}$.}
\label{fig: phi}
\end{figure}

\section{Geometry Utilization}
\label{sec: geo util}

In the previous design (Eq. \ref{eq: D design}), we discard all the embodiment configurations information in neural network modeling. However, learning 3D geometry solely from multi-view images is difficult. Embodiment configurations provide valuable cues for 3D understanding, which can improve positioning precision and robustness to viewpoint variations. Directly injecting camera extrinsics or embodiment poses as tokens, however, breaks the required invariance. We therefore seek strategies to incorporate geometric information while preserving invariance to configuration transformations.

\subsection{Relative Camera Pose Embedding}
Following GTA \cite{gta} and PRoPE \cite{prope}, we embed camera intrinsic and relative camera pose into the attention mechanism of visual tokens (feature dimension $=d$):
\begin{equation}
\begin{aligned}
    & [\mathbf{q}_i; \mathbf{k}_i; \mathbf{v}_i] = [\mathbf{W}_q; \mathbf{W}_k; \mathbf{W}_v](\mathbf{c}_i + \mathbf{W}_p \mathbf{p}_i) \\
    &\mathbf{q}_i' = \sigma(\leftindex[]^{c}_{b} \mathbf{T}_i^\top) \mathbf{q}_i, \ \mathbf{k}_i' = \sigma(\leftindex[]^b_{c} \mathbf{T}_i) \mathbf{k}_i, \ \mathbf{v}_i' = \sigma(\leftindex[]^b_{c} \mathbf{T}_i) \mathbf{v}_i \\
    &\mathbf{o}_i = \sigma(\leftindex[]^{c}_{b} \mathbf{T}_i) \sum_j \frac{\exp(\mathbf{q}_i' {}^\top \mathbf{k}_j' / \sqrt{d}) }{ \sum_k \exp( \mathbf{q}_i' {}^\top \mathbf{k}_k' / \sqrt{d} ) } \mathbf{v}_j'
\end{aligned}
\end{equation}
where $\sigma(\mathbf{T})$ compose the pose matrix $\mathbf{T}$ for $d/4$ times to build a large block-diagonal matrix:
\begin{equation}
    \sigma(\mathbf{T}) = \begin{bmatrix}
        \mathbf{T} & & \\
        & \ddots & \\
        & & \mathbf{T}
    \end{bmatrix}_{d\times d}
\end{equation}
$\mathbf{p}_i$ is the normalized camera place coordinates of $i$-th image token $\mathbf{c}_i$, and $\leftindex[]^b_{c} \mathbf{T}_i$ is its associated camera pose. 
We adopt the additive positional embedding for intrinsic and rotary positional embedding for extrinsic. This design allows geometry-aware self-attention without violating invariance constraints.

\subsection{Actions as Positional Embedding}
We generate actions with diffusion. At each diffusion timestep $k$, we add positional embedding to the noisy action tokens when cross attending to condition tokens:
\begin{equation}
\begin{aligned}
    \mathbf{q}^a &= \mathbf{W}_q ( \mathbf{W}_y \hat{\mathbf{y}}_k + \mathbf{W}_p \mathbf{p}^a_k) \\
    \mathbf{p}^a_k &= \mathcal{P}(\leftindex[]^b_c \mathbf{T}^{-1} \ \mathcal{D}(\mathbf{m}, \hat{\mathbf{y}}_k)) \\
    \mathbf{o}^a &= \mathrm{scale\_dot\_product}(\mathbf{q}^a, \mathbf{k}, \mathbf{v})
\end{aligned}
\end{equation}
where $\mathcal{P}$ projects the 3D position of future noisy action to the image space.
By converting embodiment configurations into positional encodings in image space, projected future actions provide feedback during recurrent denoising. This explicitly links actions to image patch coordinates, thereby improving positioning precision.

\begin{figure}[t]
\includegraphics[width=0.8\linewidth]{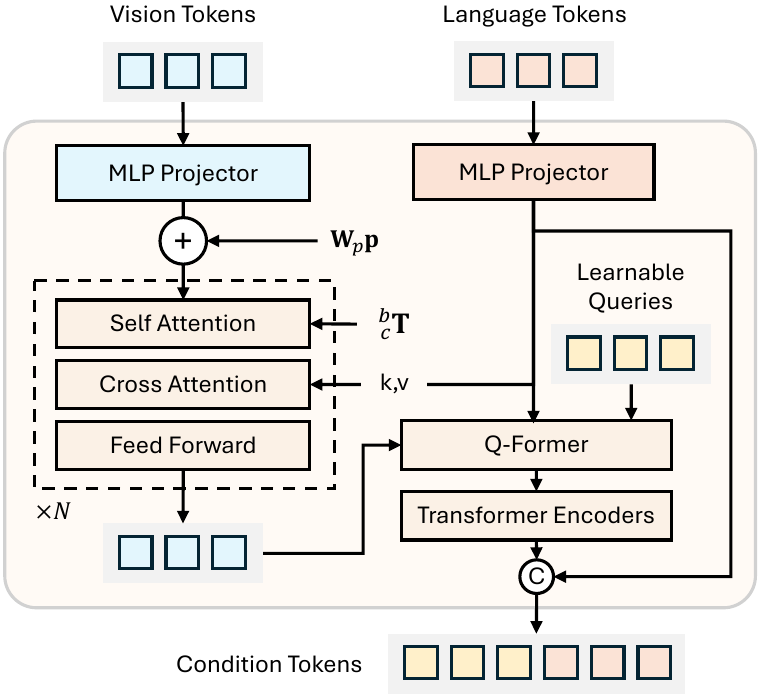}
\centering
\caption{Schematics of model $\phi_\text{VL}$. The initial vision tokens and language tokens are extracted from fronzen DINOv2 and SigLIP models. Image self attention with relative camera pose embedding is adopted to enhance geometry understanding. Image-language cross attention is adopted to enhance instruction following. Q-Former is adopted to compress task-related vision features for efficient action generation.}
\label{fig: phi_VL}
\end{figure}

\begin{figure}[t]
\includegraphics[width=0.8\linewidth]{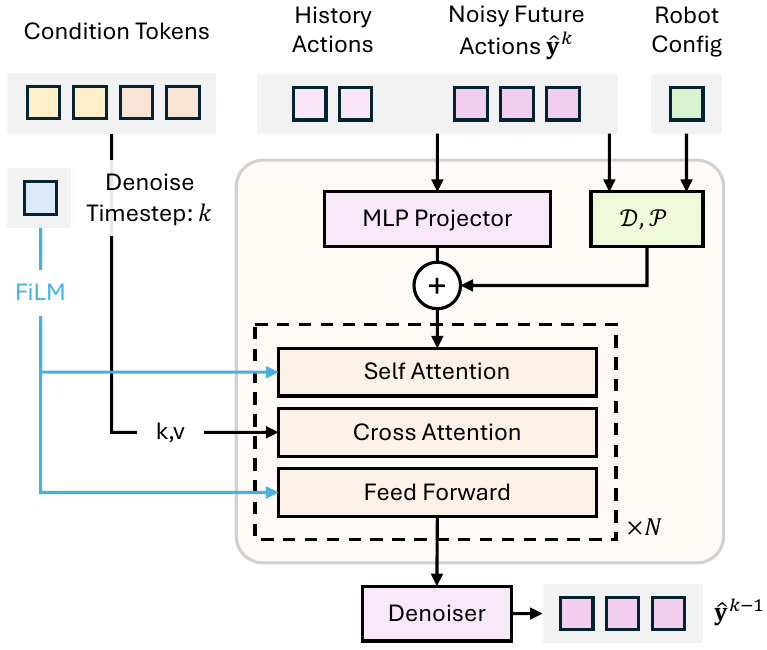}
\centering
\caption{Schematics of model $\phi_\text{A}$. We use cross attention to inject vision-language conditions and FiLM to inject denoise timestep. }
\label{fig: phi_A}
\end{figure}

\subsection{Network Structure}
Following standard VLA policy design, the policy network is decomposed into a vision-language encoder $\phi_\text{VL}$ and an action generator $\phi_\text{A}$ conditioned on $\phi_\text{VL}$.

In $\phi_\text{VL}$, we fuse the DINOv2 \cite{dinov2} features and SigLIP \cite{siglip} image features with a MLP, followed by repeated image self-attention and image-text cross-attention.
Relative camera pose embedding is injected into all the self-attention layers of vision features. A one-layer Q-Former \cite{instructblip} with 64 learnable queries is used to compress task-relevant features for action generation.

In $\phi_\text{A}$, action tokens are processed by self-attention and cross-attention with $\phi_\text{VL}$ outputs. Action positional embeddings are injected once before the first self-attention layer.
We use FiLM \cite{film} to inject denoise time step $k$ into self-attention layers and feed forward layers. The final denoised learning objective $\hat{\mathbf{y}}_0$ is decoded to $\leftindex[]^b_{e_*} \hat{\mathbf{T}}$ with $\mathcal{D}$ for execution.

Excluding the frozen vision-language backbones, our action expert consists total 122M trainable parameters, which is approximately 1/3 of the action expert size in $\pi_0$ \cite{pi0}.

\subsection{Implementation Details}

\textbf{Dataset.} We pre-train our policy on a real-world dataset (Droid \cite{droid}) and several simulation datasets (ManiSkill \cite{maniskill}, Meta-world \cite{metaworld} and self collected simulation data using IsaacSim). The policy is fine-tuned on LIBERO \cite{libero} for simulated benchmark, and self collected datasets for real-world evaluation. Details of datasets are listed in Table. \ref{tab: datasets}.

\begin{table}[tb]
\centering
\caption{Datasets used in pre-training and fine-tuning. Datasets with notation * are self-collected. The LIBERO dataset are regenerated with original action replayed and failed demonstrations removed.}

\resizebox{\linewidth}{!}{

\begin{tabular}{lllll}
\toprule
Stage                                                                              & Dataset             & Episodes & Cameras            & Robot                   \\ \midrule
\multirow{6}{*}{Pre-training}                                                      & Droid               & 78,544   & 3                  & Franka                  \\
                                                                                   & Maniskill           & 30,123   & 2                  & Franka                  \\
                                                                                   & Meta-world          & 2,500    & 3 of 7             & Sawyer                  \\
                                                                                   & Pick Place Cans*     & 2,000    & 2                  & UR5                     \\
                                                                                   & Open Drawer*         & 2,000    & 1                  & UR5                     \\
                                                                                   & Open Microwave*      & 2,000    & 1                  & UR5                     \\ \midrule
\multirow{4}{*}{\begin{tabular}[c]{@{}l@{}}Fine-tuning,\\ Simulation\end{tabular}}  & LIBERO-Spatial      & 438      & \multirow{4}{*}{2} & \multirow{4}{*}{Franka} \\
                                                                                   & LIBERO-Object       & 452      &                    &                         \\
                                                                                   & LIBERO-Goal         & 421      &                    &                         \\
                                                                                   & LIBERO-Long         & 376      &                    &                         \\ \midrule
\multirow{4}{*}{\begin{tabular}[c]{@{}l@{}}Fine-tuning,\\ Real-world\end{tabular}} & Pick Place Pepper*   & 10       & \multirow{4}{*}{2} & \multirow{4}{*}{Piper}  \\
                                                                                   & Pick Place Carrot*   & 10       &                    &                         \\
                                                                                   & Pick Place Eggplant* & 10       &                    &                         \\
                                                                                   & Table Storage*       & 30       &                    &                         \\ \bottomrule
\end{tabular}
}
\label{tab: datasets}
\end{table}

\textbf{Training.} We use AdamW as the optimizer with learning rate 1e-4 and weight decay 1e-2. We use linear learning rate warmup for 20k iterations in pre-training and 2k iterations in fine-tuning. We train all the models with batch size 32 in mixed precision (float32 and bfloat16). Pre-training runs for 600k iterations, which costs around 55 hours on single RTX 4090. We use 100-step DDPM in training and 20-step DDIM in inference.

\textbf{Learning Objective.} As shown in Eq. \ref{eq: D design stable}, the learning objective $\hat{\mathbf{y}}_r $ of policy $\phi_\theta$ lies on the $\mathrm{SE}(3)$ manifold. We use 6D representation \cite{rot6d} of rotation component. We additionally regress the normalized gripper width ($-1 =$ fully closed and $1 =$ fully open). The total action dimension is 10 (3 for translation, 6 for 6D rotation and 1 for gripper).

\begin{figure*}[t]
\flushright
\includegraphics[width=0.97\linewidth]{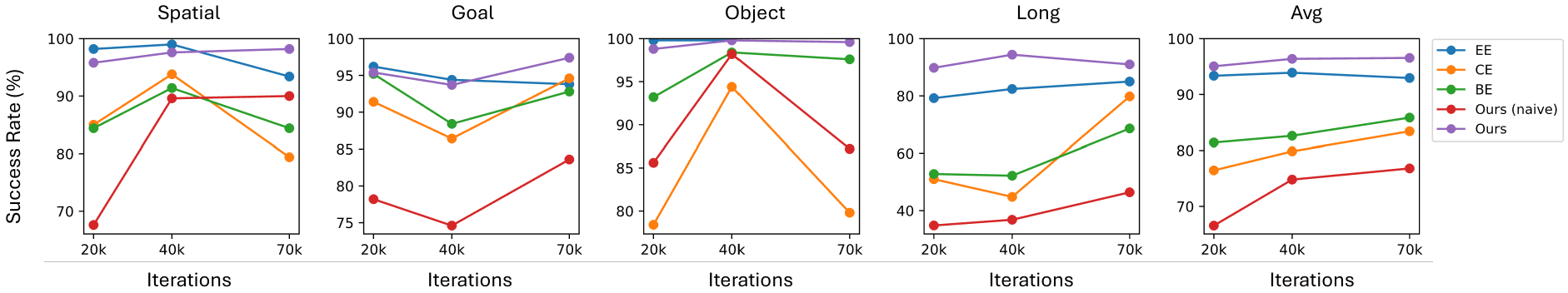}
\caption{Performance of policies with different action space on LIBERO benchmark. We evaluate policies fine-tuned for 20k, 40k and 70k iterations. Our method achieves the highest success rate on the average. All the methods except for the naive version of our implementation utilize the geometric information metioned in section \ref{sec: geo util}, which shows the embodiment configuration information is important to the performance.}
\label{fig: libero metrics}
\end{figure*}

\section{Simulation Experiments}

\subsection{Performance on LIBERO Benchmark}
After pre-training, we fine-tune our model on LIBERO \cite{libero} simulation benchmarks, including LIBERO-(Spatial, Goal, Object and Long). 
Since the original datasets lack camera extrinsics and intrinsics, we follow OpenVLA \cite{openvla} to replay the ground truth actions and record this information. Failed demonstrations are removed before fine-tuning.
We compare our fine-tuned policy with recent VLA baselines. Results in Table. \ref{tab: libero} show our policy achieves competitive or superior performance while requiring substantially lower fine-tuning cost due to more effective pre-training.
Specifically, our policy reaches its reported results with only 40k iterations (less than 5 GPU hours) fine-tuning, whereas $\pi_0$ \cite{pi0} requires $\sim$34 GPU hours, and OpenVLA \cite{openvla} / UniVLA \cite{univla} require over 300 GPU hours.

\begin{table}[tb]
\centering
\caption{Success rate (\%) and total parameters (B) of fine-tuned policies on LIBERO benchmark.}
\begin{tabular}{lllllll}
\toprule
Method  & Spatial & Goal & Object & Long & Avg & Param   \\ \midrule
OpenVLA & 84.7    & 88.4 & 79.2   & 53.7 & 76.5 & 7  \\
$\pi_0$ & 96.8    & \textbf{98.8} & 95.8   & 85.2 & 94.15 & 2.3 \\
UniVLA  & 96.5    & 96.8 & 95.6   & 92.0 & 95.2 & 8  \\
Ours    & \textbf{97.6}    & 93.7 & \textbf{99.8}   & \textbf{94.4} & \textbf{96.38} & \textbf{0.4} \\ \bottomrule
\end{tabular}
\label{tab: libero}
\end{table}

\subsection{Ablation Study of Different Action Space Design}

We also pre-train three policies with the same neural network structure as ours but with different decoder and learning objectives. Success rates versus fine-tuning iterations are shown in Fig. \ref{fig: libero metrics}.

These results highlight that our action space yields the most efficient fine-tuning, since it preserves consistent action representation across tasks regardless of embodiment configurations.
We notice EE action space and ours are both in the relative action space, achieving higher success rate than the absolute action space (BE and CE). 
This is because relative action spaces produce approximately zero-centered action distributions, whereas absolute spaces suffer from mean-shift when transferring across datasets, making fine-tuning harder and reducing positioning precision. We also observe that the naive variant of our method (discarding all embodiment configuration information) performs worst, indicating that embodiment information is essential for trajectory positioning precision.

\begin{figure}[t]
\includegraphics[width=0.8\linewidth]{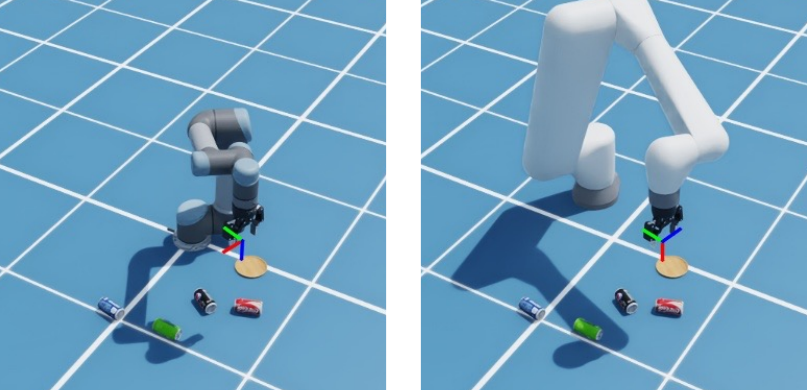}
\centering
\caption{Pick place task in simulation. The policies are first fine-tuned with sufficient data (1000 demonstrations) generated using UR5 (left), and zero-shot evaluated on the Fanuc (right), and finally fine-tuned with little data (50 demonstrations) generated using Fanuc (right). We shift both the robot base and end-effector definitions. The pose of third-person camera is randomized in demonstration and evaluation.}
\label{fig: ur5 to fanuc}
\end{figure}

\subsection{Zero/Few-shot to New Embodiment Configuration}
We further evaluate cross-embodiment generalization in a simulated pick-place task. We first fine-tune all the policies to reach their best performance with sufficient demonstrations collected by UR5. We then replace the embodiment from UR5 to Fanuc (Fig. \ref{fig: ur5 to fanuc}), which is never seen in the pre-training stage of our policy and $\pi_0$. We also shift the base pose and rotate the definition of end-effector coordinates. We measure zero-shot and few-shot performance on this new embodiment.

\begin{table}[tb]
\centering
\caption{Zero-shot and few-shot results on pick place task with a new embodiment. Our method achieves consistent success rate when transferring to the new embodiment.}
\begin{tabular}{llllll}
\toprule
Configuration     & $\pi_0$ & BE & CE & EE & Ours \\ \midrule
UR5 (original)    & 82  & 80 & 72 & \textbf{98} & \textbf{98}   \\
Fanuc (zero-shot) & 0   & 0  & 0  & 0  & \textbf{94}   \\
Fanuc (fine-tune) & 22  & 50 & 52 & 92 & \textbf{98}   \\ \bottomrule
\end{tabular}
\label{tab: ur5 to fanuc}
\end{table}

As shown in Table \ref{tab: ur5 to fanuc}, our method maintains a high success rate even under zero-shot transfer to the Fanuc arm. Although appearance differences between embodiments cause slight degradation, our equivariant design still generalizes effectively. Other action spaces completely fail in zero-shot settings. With few-shot fine-tuning, our method and EE action space recover to the comparable performance as on UR5, while others struggle to close the gap. We also found the joint space action performs the worst when adapting to new embodiment with few demonstrations.

\section{Real-world Experiments}

We conduct real-world evaluations on two tasks, pick-place and table storage, using the Agilex Cobot platform equipped with Piper arms (Fig. \ref{fig: teleop}). For each task, we collect 30 teleoperated demonstrations. Both our policy and $\pi_0$ are fine-tuned on this data for 40k iterations with a batch size of 32. Notably, the Agilex Cobot configuration is unseen during our pre-training, while $\pi_0$ was pre-trained partially on mobile Trossen and mobile ARX, which share similar embodiment configuration with Agilex Cobot. In all experiments, $\pi_0$ is fine-tuned with joint-space actions, consistent with its pre-training setup for the mobile Trossen/ARX. Sampled successful roll-outs are visualized in Fig. \ref{fig: success_cases}.

\begin{figure}[t]
\includegraphics[width=\linewidth]{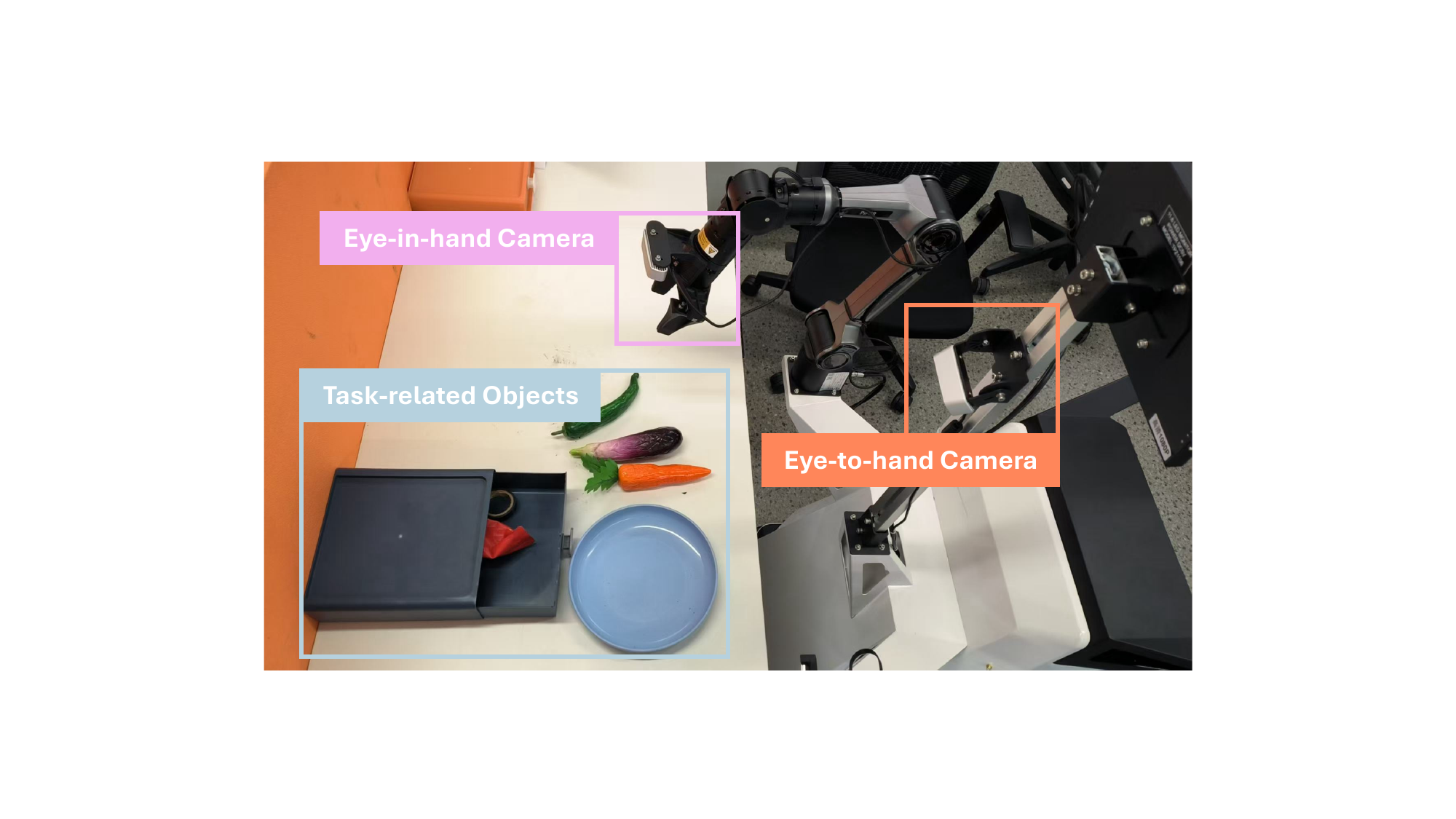}
\centering
\caption{Data collection and policy evaluation platform in real-world. We use images from two calibrated cameras as observations.}
\label{fig: teleop}
\end{figure}

\begin{figure}[t]
\includegraphics[width=\linewidth]{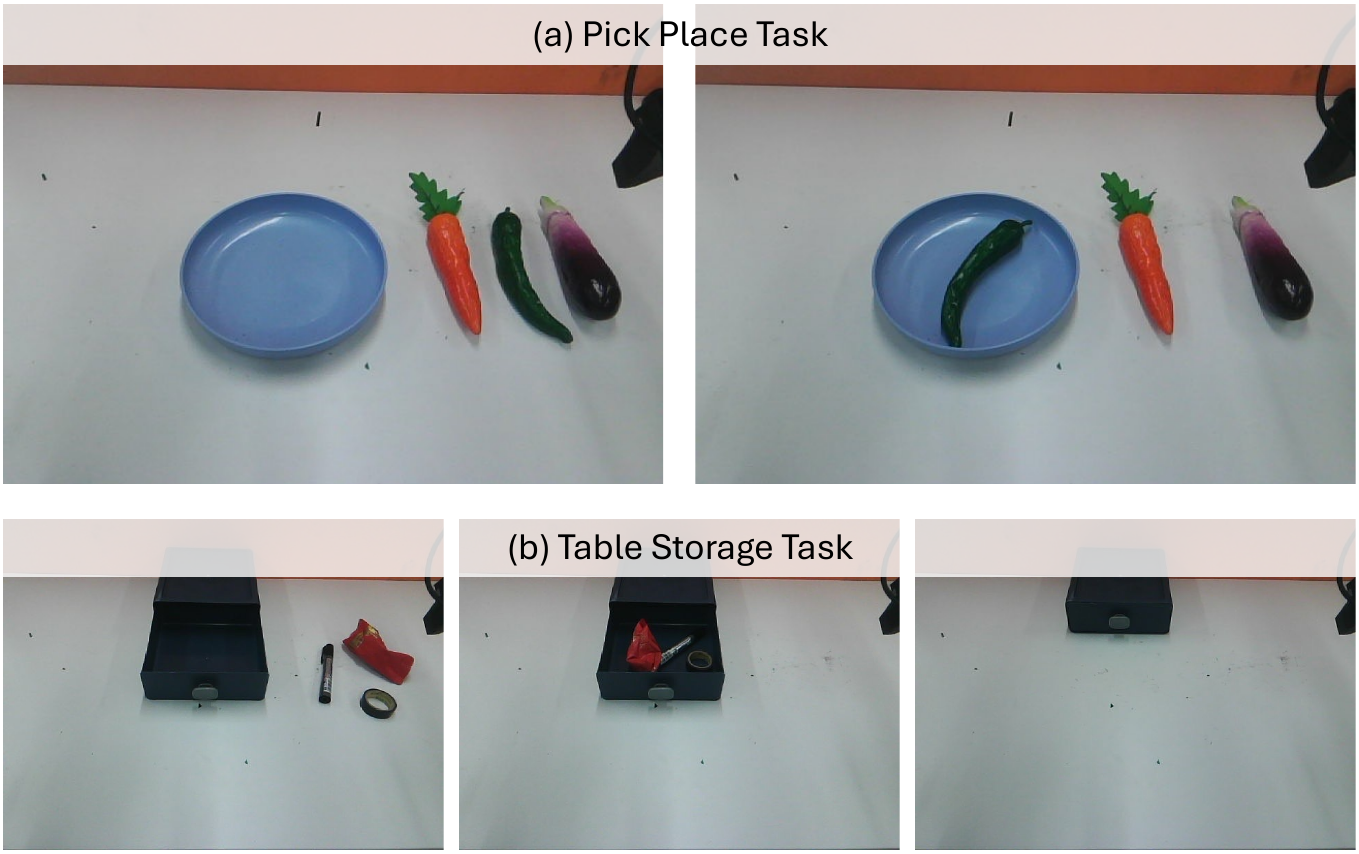}
\centering
\caption{Tasks in real-world experiments. (a) The initial state (left) and the desired final state (right) of pick place task. The example prompt is: \textit{pick up the pepper and place it on the plate}. (b) Three stages of completing the table storage task. Left to middle: pick all the items on the table and place them in the box. Middle to right: push to close the box.}
\label{fig: real_tasks}
\end{figure}

\begin{figure}[tb]
\includegraphics[width=\linewidth]{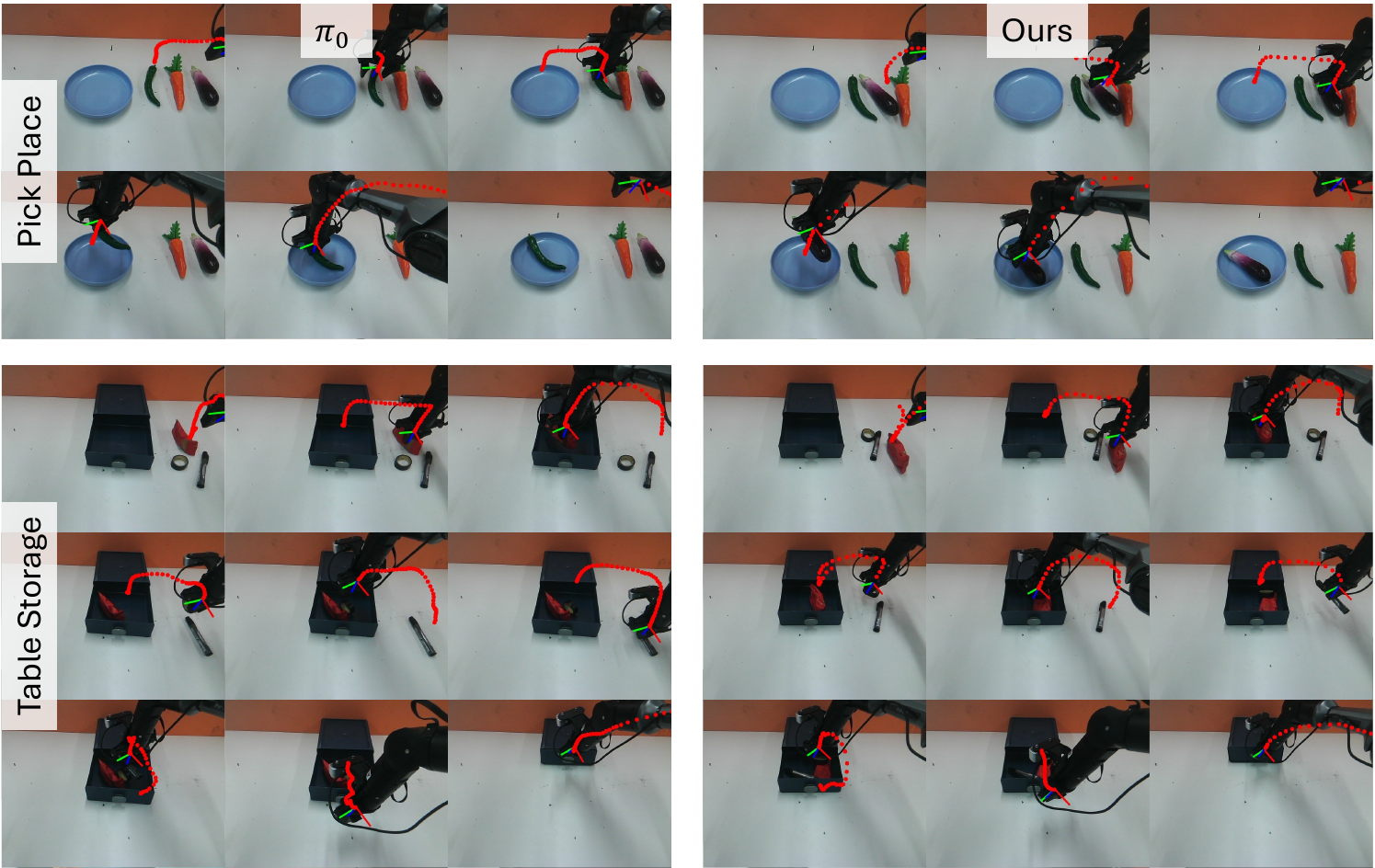}
\centering
\caption{Example success roll-outs of $\pi_0$ and our policy on pick place task and stable storage task.}
\label{fig: success_cases}
\end{figure}

\subsection{Pick-Place Task}
\textbf{Task Description.} The prompt is: \textit{Pick up the \{object\} and place it on the plate}, where the object is selected from pepper, carrot and eggplant. The example initial state and the desired final state is visualized in Fig. \ref{fig: real_tasks}a. We evaluate each policy for 60 trials. For the first 30 trials, there exists only the plate and the desired object. For the rest 30 trials, we place all the objects on the table acting as distractors.

\begin{table}[tb]
\centering
\caption{Success rate of real-world pick place task.}
\begin{tabular}{llllll}
\toprule
Distractor         & Method & Pepper & Carrot & Eggplant & Avg   \\ \midrule
\multirow{2}{*}{\XSolidBrush} & $\pi_0$    & 7/10   & 9/10   & 10/10    & 26/30 \\
                   & Ours   & 10/10  & 10/10  & 10/10    & 30/30 \\ \midrule
\multirow{2}{*}{\Checkmark} & $\pi_0$    & 4/10   & 4/10   & 3/10     & 11/30 \\
                   & Ours   & 9/10   & 10/10  & 10/10    & 29/30 \\ \bottomrule
\end{tabular}
\label{tab: pick place real}
\end{table}

\textbf{Result Analysis.} As shown in Table. \ref{tab: pick place real}, both $\pi_0$ and our policy achieve high success rate without distractors. 
However, $\pi_0$ degrades substantially when distractors are introduced, reflecting its weaker grounding between language and action given limited data for fine-tuning. In contrast, our policy maintains nearly perfect performance. Fig. \ref{fig: add_distractor_result} visualizes how distractors influence the trajectory planned by $\pi_0$, while ours remains almost unaffected.

\begin{figure}[tb]
\includegraphics[width=\linewidth]{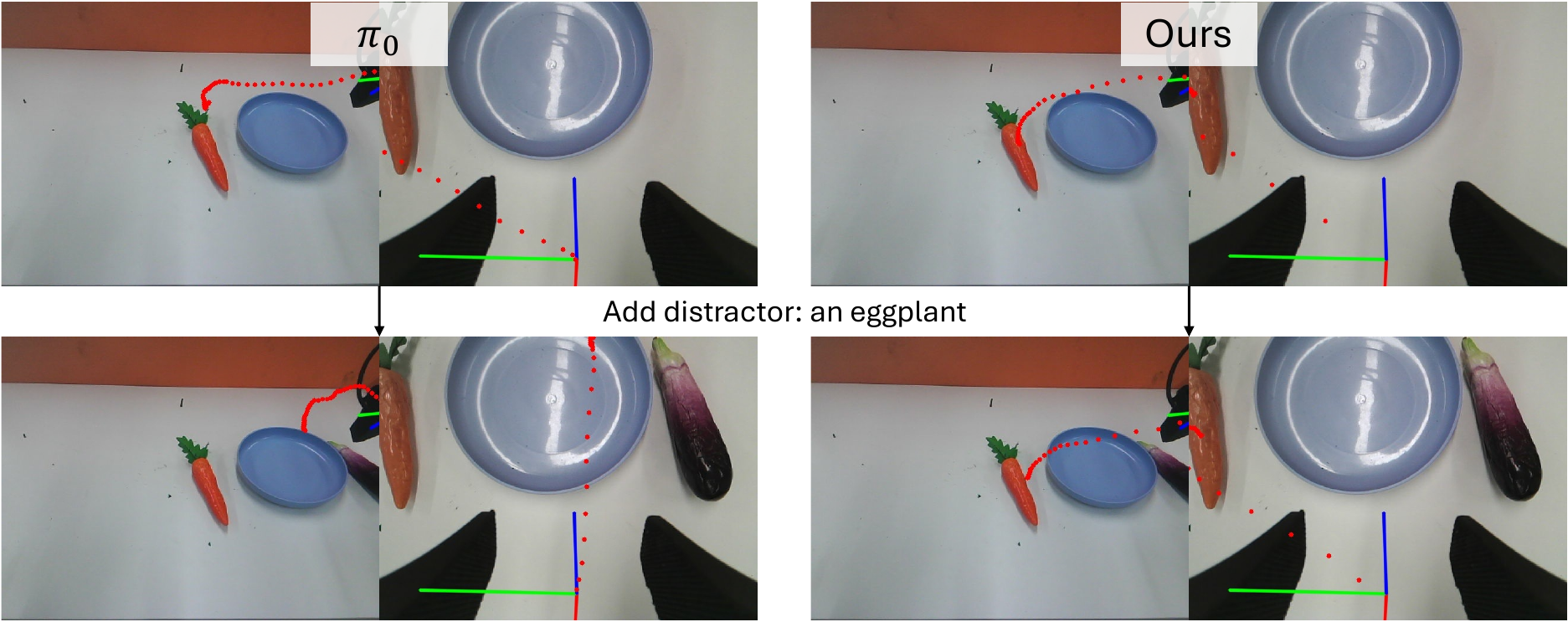}
\centering
\caption{Effect of distractors on planned trajectories (shown as red dots). The prompt is: \textit{pick up the carrot and place it on the plate}. Both $\pi_0$ and our policy plan correct trajectories without distractors. However, When an eggplant is placed beside the plate, $\pi_0$ is easily distracted, while ours remains robust.}
\label{fig: add_distractor_result}
\end{figure}

\subsection{Table Storage Task}
\textbf{Task Description.} The prompt is: \textit{Put the items on the table into the storage box and close the box}. The robot needs to pick up all the items (including a marker pen, tape and snack) on the table, place them in the box, and finally push the box closed. The sample initial state, intermediate state and desired final state is visualized in Fig. \ref{fig: real_tasks}b. We evaluate the policies for 40 trials. For the first 30 trials, we only perturb the position of the box and the items. For the rest 10 trials, we perturb the pose of robot base, resulting in different camera viewpoints (Fig. \ref{fig: pose_disturb}).

\begin{figure}[tb]
\includegraphics[width=\linewidth]{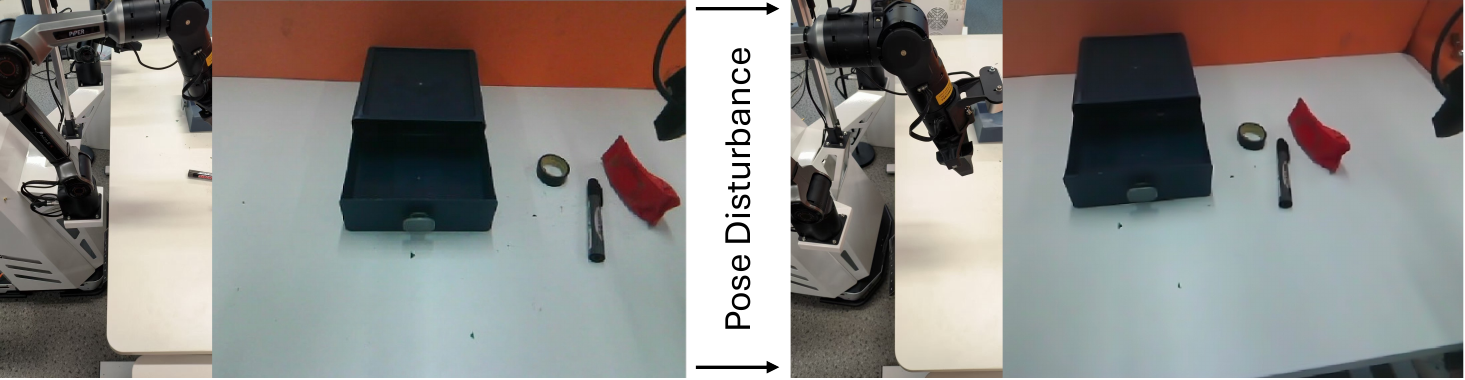}
\centering
\caption{Robot base pose disturbance is applied in the last 10 trials evaluation of table storage task. }
\label{fig: pose_disturb}
\end{figure}

\textbf{Result Analysis.} As completing the table storage requiring multiple pick-place and push attempts, we report the success rate of each step. 
Results (Table. \ref{tab: table_storage}) show that our policy consistently outperforms $\pi_0$, particularly under base pose disturbance. The robustness comes from representing relative end-effector motions in the camera frame, whereas $\pi_0$ overfits to embodiment-specific dynamics and fails to close the box.

\begin{figure}[tb]
\includegraphics[width=\linewidth]{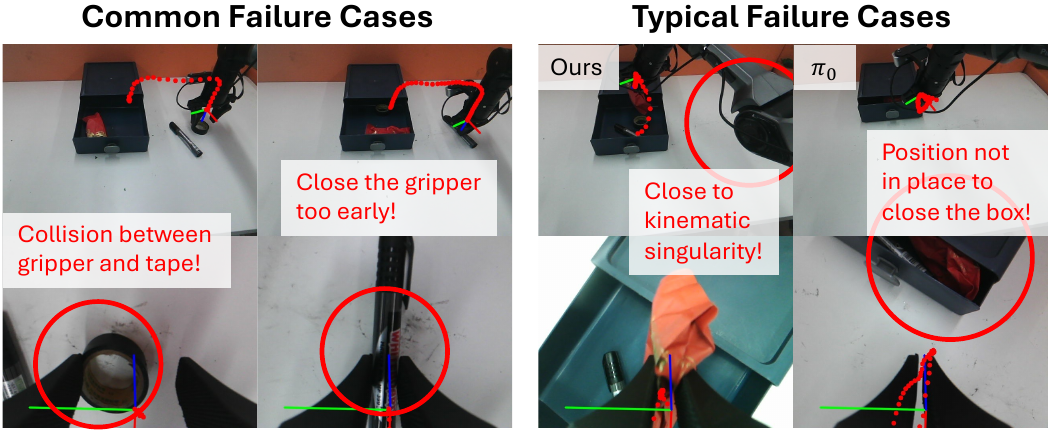}
\centering
\caption{Common and typical failure cases of our policy and $\pi_0$.}
\label{fig: failure cases}
\end{figure}

\begin{table}[tb]
\centering
\caption{Success rate of real-world table storage task.}
\begin{tabular}{lllllll}
\toprule
Disturbance       & Method & Pen   & Snack & Tape  & Box   & Overall \\ \midrule
\multirow{2}{*}{\XSolidBrush} & $\pi_0$    & 27/30 & 28/30 & 24/30 & 22/30 & 20/30   \\
                   & Ours   & 29/30 & 28/30 & 26/30 & 27/30 & 24/30   \\ \midrule
\multirow{2}{*}{\Checkmark} & $\pi_0$    & 6/10  & 8/10  & 5/10  & 0/10  & 0/10    \\
                   & Ours   & 9/10  & 10/10 & 10/10 & 9/10  & 9/10    \\ \bottomrule
\end{tabular}
\label{tab: table_storage}
\end{table}

\textbf{Failure Case Analysis.} As shown in Fig. \ref{fig: failure cases}, we found the most common issue arises from inaccurate Z-axis localization using the eye-in-hand camera, especially when grasping slender objects like the pen. The inaccurate XY-axis localization may cause collision between the gripper and the tape. 
$\pi_0$ frequently fails at the box-closing step due to its reliance on joint-space actions, which cannot generalize well with pose disturbance. Our policy occasionally fails under unreachable poses or kinematic singularities, though it is generally more robust overall.

\section{Limitations}
Although our action space and network architecture design fully leverage the relative pose representation in camera space for efficient policy fine-tuning, it requires camera extrinsic and intrinsic recorded. 
In practice, many large-scale real-world robotic datasets provide high-quality images, detailed language annotations, and diverse environments, but often lack reliable camera calibration, or even doesn't record camera parameters.
This limitation currently prevents internet large-scale pre-training of our policy. However, we have seen more researchers realizing the importance of camera parameters in data collection and policy learning. 
In addition, emerging approaches such as VGGT \cite{vggt} demonstrate the potential of recovering pseudo camera poses from images, which may help alleviate this constraint in future work.

\section{Conclusion}
In this paper, we propose an action space invariant to embodiment configuration transformation and the corresponding action decoder equivariant to the transformation. 
Extensive simulation experiments on the LIBERO benchmark demonstrated that our policy achieves state-of-the-art performance with significantly reduced fine-tuning cost. Moreover, our approach enables efficient transfer to unseen embodiments, achieving high success rates even under zero- and few-shot settings. Real-world experiments further validated the effectiveness and robustness of our method in manipulation tasks with distractors and embodiment perturbations. Our proposed action space and architecture design is a feasible choice for effective cross-embodiment pre-training and efficient fine-tuning.

\balance
\bibliographystyle{IEEEtran}
\bibliography{main}

\clearpage

\end{document}